\documentclass[preprint]{article}
\newif\ifpreprint
\preprinttrue

\usepackage{neurips_2020}

\usepackage[utf8]{inputenc} 
\usepackage[T1]{fontenc}    
\usepackage{hyperref}       
\usepackage{url}            
\usepackage{booktabs}       
\usepackage{amsfonts}       
\usepackage{nicefrac}       
\usepackage{microtype}      
\usepackage{microtype}
\usepackage{graphicx}
\usepackage{tikz}
\usepackage{subfig}
\usepackage{wrapfig}
\usepackage{color}

\author{%
  Kai Br{\"u}gge\\\
  Dept.\ of Computer Science\\
  University of Copenhagen\\
  \texttt{kai.brugge@gmail.com} \And
   Asja Fischer\\
  Faculty of Mathematics\\
  Ruhr University Bochum\\
  \texttt{asja.fischer@rub.de}\\ \And
    Christian Igel\\
  Dept.\ of Computer Science\\
  University of Copenhagen\\
  \texttt{igel@di.ku.dk} \\
}

\usepackage{hyperref}
\usepackage{amsmath,amsthm,amssymb}
\newtheorem{lemma}{Lemma}
\newtheorem{corollary}{Corollary}
\newtheorem{theorem}{Theorem}

\def\vec#1{\mathchoice{\mbox{\boldmath$\displaystyle#1$}}
{\mbox{\boldmath$\textstyle#1$}}
{\mbox{\boldmath$\scriptstyle#1$}}
{\mbox{\boldmath$\scriptscriptstyle#1$}}}
\newcommand{\x}{\vec{x}}
\newcommand{\y}{\vec{y}}
\newcommand{\xh}{{f_i(\vec{x})}}
\newcommand{\degp}{\operatorname{deg}^{+}}
\newcommand{\degm}{\operatorname{deg}^{-}}

\newcommand{\ci}[1]{\textcolor{magenta}{#1}}
\newcommand{\af}[1]{\textcolor{blue}{#1}}

\newcommand{\kb}[1]{\textcolor{purple}{#1}}

\renewcommand{\kb}[1]{#1}
\renewcommand{\ci}[1]{#1}
\renewcommand{\af}[1]{#1}

\newcommand{\xbar}[1]{{\overline{#1}}}
\newcommand{\fli}{{f_{\le i-1}}}
\newcommand{\fgi}{{f_{\ge i+1}}}

\usepackage{scalerel,lmodern}
\input pdf-trans
\newbox\qbox
\def\usecolor#1{\csname\string\color@#1\endcsname\space}
\newcommand\outline[1]{\leavevmode%
  \def\maltext{#1}%
  \setbox\qbox=\hbox{\maltext}%
  \boxgs{Q q 2 Tr \thickness\space w 0 0 0 rg 0 G}{}%
  \copy\qbox%
}
\newcommand\mathbbf[2][.2]{%
  \def\thickness{#1}%
  \ThisStyle{\outline{$\mathbf{\SavedStyle#2}$}}%
}

\begin{document}

\title{On the convergence of the Metropolis algorithm with fixed-order updates for multivariate binary probability distributions}

\maketitle

\begin{abstract}
The Metropolis algorithm is arguably the most fundamental Markov chain Monte Carlo (MCMC) method.  But the algorithm is not guaranteed to converge to the desired distribution in the case of multivariate binary distributions (e.g., Ising models or stochastic neural networks such as Boltzmann machines) if the variables (sites or neurons) are updated in a fixed order, a setting commonly used in practice. The reason is that the corresponding Markov chain may not be irreducible. We propose a modified Metropolis transition operator that behaves almost always identically to the standard Metropolis operator and prove that it ensures irreducibility and convergence to the limiting distribution in the multivariate binary case with fixed-order updates. The result provides an explanation for the behaviour of Metropolis MCMC in that setting and closes a long-standing theoretical gap. We experimentally studied the standard and modified Metropolis operator for models were they actually behave differently. If the standard algorithm also converges, the modified operator exhibits similar (if not better) performance in terms of convergence speed.
\end{abstract}

\section{Introduction}
Markov Chain Monte Carlo (MCMC) algorithms address the problem of
sampling from a probability distribution $p$ by constructing a Markov
chain with stationary distribution equal to $p$. Recording a state of
the chain after running it for some time replaces sampling from $p$.
It has to be ensured that the Markov chain is ergodic, that is,
converges to the stationary distribution irrespective of the starting
state.
The Metropolis algorithm \citep{metropolis1953} established the field
of MCMC and is still widely used.  However, for the basic scenario of
multivariate binary distributions -- as for example given in the case of
Ising models or Boltzmann machines -- and
state-updates that consider these variables in a fixed order, it is
well known that the Metropolis algorithm may not converge to the
desired distribution (see, e.g., \citealp{friedberg1970}), because 
the Markov chain induced can be reducible.

\cite{brugge2013}  have suggested a slightly
modified Metropolis algorithm for sampling of restricted Boltzmann machines.
We extend their work and prove that the modified algorithm induces ergodic Markov chains for all multivariate binary
distributions (in the non-binary case the convergence problem does
not occur).

As a corollary we give a sufficient condition for the standard Metropolis algorithm
to converge. For many classes of models, this condition is fulfilled almost surely. This theoretically justifies the use of the standard Metropolis algorithm with fixed-order updates.

In the next section, we state our main result, which is proven in
Section~\ref{sec:proof}. Section~\ref{sec:examples} provides 
numerical experiments before we conclude in Section~\ref{sec:conclusions}.

\section{Main result}
We consider the important case where $p$ is an $n$-dimensional
multivariate distribution with full support over a finite set
$\Omega^n$ with binary $\Omega$, covering Ising models
\citep{ising1925beitrag,HistoryLenzIsing} and stochastic neural
networks such as restricted Boltzmann machines
\citep{smolensky,HintonCD,fischer:13}. The transition probabilities of
a (homogenous) Markov chain with state-space $\Omega^n$ can be defined
by a transition operator $\vec{T}$ representing the probabilities
$\vec{T}(\vec x\to\vec y)$ of going from state $\x$ to $\y$ in one
step.

If $p$ is a multivariate distribution, it is a common approach to
consider a transition operator that is defined as a concatenation of operators of
the form $\vec{T}=T_n\circ\dots\circ T_2\circ T_1$, where operator
$T_i$ can only change the $i$-th variable (e.g., as in Gibbs sampling,
\citealp{geman84}). We refer to the typical case of always applying
the $n$ operators in the same order as \emph{fixed-order} updates.

For the standard \emph{Metropolis algorithm}
the transition operator for the $i$-th variable is defined as:\footnote{For binary $\Omega$ we need not
  distinguish between proposal distribution and acceptance function,
  because flipping a state covers all proposal distributions.}
\begin{equation}\label{eq:standard}
  T_i(\x\to\xh)=
  \begin{cases}
  1 & \text{if\ } p(\x) \le p(\xh)\\
  \frac{p(\xh)}{p(\x)} & \text{if\ } p(\x) > p(\xh)\\
  \end{cases}
\end{equation}
for $\x=(x_1, \dots, x_{i-1}, x_i, x_{i+1},\dots, x_n)$ and $\xh=(x_1,
\dots, x_{i-1}, \bar{x}_i, x_{i+2},\dots, x_n)$, where $\bar{x}_i$ is
the flipped value of $x_i$.
We have $T_i(\x\to\x)= 1 - T_i(\x\to\xh)$ and $T_i(\x\to\y)=0$ for
$\y\in\Omega^n\setminus\{\x,\xh\}$.
When using this operator, the Markov chain may not converge
to $p$, because it may not be \emph{irreducible}.

A Markov chain is irreducible if one can get from any state to any
other in a finite number of transitions.  Irreducibility is necessary for the chain to always converge to a unique stationary distribution.  For fixed-order
updates, the transition operators of the Metropolis algorithm do not
necessarily lead to an irreducible Markov chain and the chain may not converge,
potentially leading to a failure of the MCMC sampling algorithm. Examples are
given in Section~\ref{app:counter} in the supplementary appendix.

Updating the variables in  random order simplifies the theoretical analysis and guarantees an irreducible Markov chain. In practice, though, a fixed order is
usually preferred, because it can be implemented more
efficiently. Additionally, fixed-order updates are usually assumed to
lead to faster convergence as they ensure that each variable is
updated equally often and no variable is neglected for a longer
time. Still, they can also lead to slower convergence or even a non-ergodic chain
\citep{neal1993}.

For multivariate binary distributions, the Metropolis algorithm is
similar to  Gibbs sampling \citep{geman84}, in fact Gibbs sampling
can be seen as \kb{a} Metropolis algorithm with a different proposal distribution
or acceptance function (the Boltzmann acceptance function). 
A number of papers address the question
of which of those two methods is preferable, see, e.g., \citet{peskun1973,cunningham1974, frigessi1993, peskun1981}. \citet{neal1993} concludes: ``The issues still remain unclear, though it appears that common opinion favors using the Metropolis acceptance function in most circumstances.''

\cite{brugge2013}  proposed a slightly modified Metropolis operator for restricted Boltzmann machines, which we generalize to arbitrary binary distributions.  The  modified Metropolis transition operator only differs in  the case when the current
and proposed state have the same probability:
\begin{equation}\label{eq:modified}
  T_i(\x\to\xh)=
  \begin{cases}
  1 & \text{if\ } p(\x) < p(\xh)\\
  \frac{p(\xh)}{p(\x)} & \text{if\ } p(\x) > p(\xh)\\
  \frac{1}{2} & \text{otherwise}
  \end{cases} \enspace.
\end{equation}
This modification ensures that $p$ remains a stationary distribution,
which is straight-forward to show by proving detailed balance.
In the next section, we prove that for this operator it holds:
\begin{theorem}\label{thm:main}
Let $p$ be a distribution with full support over $\Omega^n$ for binary
$\Omega$ and $n\ge 1$.
The Markov chain induced by the modified Metropolis operator
\eqref{eq:modified} and fixed-order updates is
irreducible and aperiodic (and therefore ergodic). 
\end{theorem}

The derivation of this theorem relies on a novel proof strategy (construction of graphs where cycles correspond to contradictions and doing induction on these graphs).
We use tools that, to our knowledge,  have not been applied in that form to the analysis of MCMC algorithms. It would have been, for example, not possible to follow the approach by \cite{brugge2013}, which requires that the graph of the MRF is bipartite, while our proof is valid for all binary models (e.g., general Boltzmann machines).

From Theorem~\ref{thm:main} follows a sufficient condition under which the standard
Metropolis algorithm does not differ from the modified
version and is therefore safe to use:
\begin{corollary}
Let $p$ be a distribution with full support over $\Omega^n$ for binary
$\Omega$ and $n\ge 1$. If
$p((x_1, \dots, x_{i-1}, x_i, x_{i+2}, x_n))\neq p((x_1,
\dots, x_{i-1}, \bar{x}_i, x_{i+2},\dots, x_n))$ for all
$i=1,\dots,n$, then the 
Markov chain induced by the standard Metropolis operator
\eqref{eq:standard} and fixed-order updates is
irreducible and aperiodic. 
\end{corollary}

For many use cases, for example spin glasses, where coupling strengths are
drawn from a continuous distribution, or Boltzmann machines, where
weights are initialized randomly,  the condition in Corollary \ref{eq:standard}
holds almost always.
Ising models with uniform coupling strength are an important case
where  the condition does not hold (see supplementary Appendix~\ref{app:counter}).
We suggest to use the modified Metroplis operator in cases where the
standard operator would not lead to an irreducible chain.
Section~\ref{sec:examples} demonstrates the performance of the new operator empirically.


\section{Proof of main result}\label{sec:proof}
To prove Theorem~\ref{thm:main},
we use a basic theorem
  (e.g., see \citealp{billingsley1995measure}, we refer to  \citealp{hobert07}, for a proof):
\begin{theorem}\label{thm:closedset}
A set $C\subseteq \Omega^n$ is \emph{closed} given a Markov chain if for all
$\x \in C:
\sum_{\y\in C}\vec{T}(\x\to\y)= 1$
(i.e., once the
chain enters $C$ it cannot leave). A Markov chain is irreducible if
the only closed subset is $C=\Omega^n$.
\end{theorem}
We show irreducibility by proving  that there exists
no proper subset  of $\Omega^n$ that cannot be left and then applying Theorem~\ref{thm:closedset}.
The proof involves the following steps. We first establish a
relation between possible transitions between states that differ only
in one variable $i$ and the probabilities of
states resulting from flipping all variables with a higher (or lower)
index 
than $i$ including and excluding $i$.
Then we show via
contradiction that all singleton subsets of the state space cannot be closed.
Then we show the same for sets with more than one element. To this end, we map each subset to a graph that contains
cycles if the subset cannot be left and prove
by induction that for all
proper subsets $S$ the graph $G(S)$ contains cycles and therefore the Markov chain is irreducible. Showing aperiodicity is straightforward.


\paragraph{Basic lemma}

Let us denote the state $(\bar{x}_1, \dots, \bar{x}_n)$, where all
sites (variables) are flipped to their opposite values, by
$\xbar{\vec{x}}$.  We denote the state
$(x_1, \dots, x_{i-1}, \bar{x}_i, x_{i+1}, \dots, x_n)$, where only
the $i$-th site is flipped, by $f_i(\vec{x})$, the state where the first
$i$ sites are flipped by $f_{\leq i}(\vec{x})=(\bar{x}_1, \dots, \bar{x}_i, x_{i+1}, \dots x_n)$, and
the state where the last
$i$ sites are flipped by  $f_{\geq i}= (x_1, \dots, {x}_{i+1}, \bar{x}_i, \dots, \bar{x}_n)$.
We define the special boundary cases
\begin{equation}\label{eq:boundary}
  f_{\leq 0}(\x) = f_{\geq n+1}(\x) =\x\enspace.
\end{equation}
If a state $\vec{y}$ can be reached from a state $\vec{x}$ in an arbitrary number of steps, we  write $\vec{x} \rightarrow \vec{y}$.
If  $\vec{y}$ cannot be reached from $\vec{x}$ in any numbers of steps, we  write $\vec{x} \not\rightarrow \vec{y}$.

The following lemma establishes a relationship between properties of the stationary distribution $p$ and {impossible} transitions of the Markov chain.  
\begin{lemma}
\label{lemma1}
For all $i=1,\dots,n$:
\begin{align*}
\vec{x} \not\rightarrow f_i(\vec{x})  &  \Rightarrow   p(f_{\leq i-1}(\vec{x})) < p(f_{\leq i}(\vec{x}))    \tag{a} \label{lemma1a} \\     
\vec{x} \not\rightarrow f_i(\vec{x})  &  \Rightarrow   p(f_{\geq i}(\vec{x}))   < p(f_{\geq i+1}(\vec{x}))				      \tag{b}\label{lemma1b}
\end{align*} 

\end{lemma}

\begin{proof}
For the proof of \eqref{lemma1a} we assume
\begin{equation}
p(f_{\leq i-1}(\vec{x})) \geq p(f_{\leq i}(\vec{x})) \label{assumption}
\end{equation}
and show that $\vec{x} \rightarrow f_i(\vec{x})$ follows. We do this
by constructing a chain of valid transitions from $\vec{x}$ to
$f_i(\vec{x})$. Note, that it is not sufficient to show that one can
reach $f_i(\vec{x})$ by applying only part of the transition operators $T_i \circ \dots \circ T_2 \circ T_1$, for $i < n$,
but one has also to show  
that it is  possible to stay in this state until the
completion of the full update step (i.e., under the remaining transition operators $T_n \circ \dots \circ T_{i+1}$).

Flipping a site is always possible,  {because} $0 < p(f_{\leq i}(\vec{x})) / p(f_{\leq i-1}(\vec{x}))$,
{since} we assume $p$ having full support. Thus, for
$i>1$ we can transition from $\vec{x}$ to {$ f_{\leq i-1}(\vec{x})
  $} in the first
$i-1$ partial steps. We can then stay in this state in the
$i$-th partial transition step by assumption \eqref{assumption} (we
need the assumption because staying in a state is only possible if it
has larger or equal probability than the newly proposed state).\footnote{The ability to stay in a certain state is the one point in the proof that works only for the modified
 Metropolis algorithm, but not the classical Metropolis
 algorithm. With the classical algorithm it is not possible to stay in
 the current state if $p(\fli(\vec{x})) = p(f_{\leq i}(\vec{x}))$,
 with the modified Metropolis operator it is.
{This also explains why the} proof carries over to the classical Metropolis algorithm
if there does not exist a $\vec{x} \in \Omega^n$ with $p(\fli(\vec{x})) = p(f_{\leq i}(\vec{x}))$.}
The latter argument also proves $\vec{x} \to f_{\leq i-1}(\vec{x})
  $ for $i=1$, which corresponds to staying in $\x$, see \eqref{eq:boundary}.

We continue to flip all the sites {with the partial transition
  operators} and end up in $\fgi(\fli(\vec{x})) =
\xbar{f_i(\vec{x})}$, the state where all sites but the $i$-th site are
flipped, at the end of the complete update step (i.e., after all $n$
partial update steps). We can then transition to $f_i(\vec{x})$ by
flipping all sites {in another complete update step}. Thus, we have
{created} a path from $\vec{x}$ to $f_i(\vec{x})$ using two full
update steps, which contradicts $\vec{x} \not\rightarrow f_i(\vec{x})$
in {\eqref{lemma1a}}.

To prove \eqref{lemma1b}, we 
analogously  construct a path involving two full update steps 
by first flipping all sites and then flipping all sites but the $i$-th.
\end{proof}

\paragraph{{Singleton subsets cannot be closed}}

Now we show that a singleton subset of the discrete state space cannot form a
closed set:

\begin{lemma} \label{lemma:singeltonSet}
Let $S=\{\x\} \subset \Omega^n$. It holds
$\exists \y \in \Omega^n \setminus  S: \x \rightarrow \y \enspace.$
\end{lemma}

\begin{proof}
We prove {the claim} by showing that the negation
\begin{equation} \label{negation}
\forall \y \in \Omega^n \setminus S: \x \not\rightarrow \y
\end{equation}
leads to a contradiction. 
Assuming \eqref{negation},
it {follows that} all states which differ from $\x$ in only one site {cannot be reached from $\x$, that is}:
\begin{equation}
\x \not\rightarrow f_i(\x)\text{ for }i=1,\dots,n \enspace.
\end{equation}
{Now,} from Lemma~\ref{lemma1} (a)  {follows that for $i=1,\dots,n: p(f_{\leq i-1}(\vec{x})) < p(f_{\leq i}(\vec{x})) $.}
%
%
From these inequalities we can construct {the following} sequence
\[
p(\vec{x}) < p(f_{\leq 1}(\vec{x})) < ....  <  p(f_{\leq n-1}(\vec{x})) < p(\xbar{\vec{x}}) \enspace.
\]
From Lemma~\ref{lemma1} (b) equivalently {follows}  a second sequence
\[
p(\xbar{\vec{x}}) < p(f_{\geq 2}(\vec{x})) < \dots < p(f_{\geq n}(\vec{x})) < p(\vec x) \enspace.
\]
Together these two 
{sequences} of inequalities lead to the contradiction $p(\vec x) < p( \vec x)$.
\end{proof}

\paragraph{Reduction to a graph problem}

Next, we  generalize  Lemma~\ref{lemma:singeltonSet} to arbitrary subsets of
$\Omega^n$. That is, we want to prove that
for all $S \subset \Omega^n$, \af{$\exists \vec x \in \Omega^n$ and}  $\exists \vec y \in \Omega^n\setminus S$ such that $\x \rightarrow \y$ and thus $S$ cannot be a closed set. The proof follows a similar line of thoughts as  the proof of  Lemma~\ref{lemma:singeltonSet}. 
We show that assuming the contrary, namely that there exists a $S \subset \Omega^n$ such that
\begin{equation}
\forall \x \in S: \forall  \y \in \Omega^n \setminus S:  \x \not\rightarrow  \y 
\end{equation}
and therefore specifically
\begin{equation}
\label{eq:proof.graph}
\forall \x \in S: \forall  f_i(\x) \in \Omega^n \setminus S:  \x \not\rightarrow     f_i(\x) \enspace,
\end{equation}
leads to a contradiction. Since reasoning about 
sequences of inequalities gets complicated when dealing with larger sets,
we reduce the problem to a graph problem via  mapping each subset $S$
  to a graph $G(S)$, such that
each inequality resulting from applying Lemma~\ref{lemma1} to
statement \eqref{eq:proof.graph} corresponds to an edge in the graph,
and a contradiction to statement \eqref{eq:proof.graph} arises if the
graph contains cycles.

\begin{lemma}\label{lem:reduction}
  Let $S \subset \Omega^n$ and let
  $G(S) = (\Omega^n, E(S))$ be defined as the
directed graph 
  with
edge set
\begin{equation}
E(S) = \lbrace \big( f_{\leq i-1}(\vec{x}), f_{\leq i}(\vec{x}) \big),
       \big(f_{\geq i}(\vec{x}), f_{\geq i+1}(\vec{x})\big) \,|\,  \x \in S \wedge f_i(\x) \in \Omega^n \setminus S \rbrace\enspace.
     \end{equation}
     If  $G(S)$ contains a cycle, then at least one state $f_i(\x)$ outside $S$ can be reached from some $\x \in S$, that is
\begin{equation}
\label{eq:proof.graph.lem}
\exists \x \in S: \exists  f_i(\x) \in \Omega^n \setminus S:  \x \rightarrow     f_i(\x) \enspace.
\end{equation}
   \end{lemma}
   \begin{proof}
    Assume that for all $\x \in S$ we have $ \forall f_i(\x) \in
       \Omega^n \setminus S: \x \not\rightarrow  f_i(\x)$.
   Then Lemma~\ref{lemma1} states that  $(\x,\y) \in E(S)$ implies
   the \emph{strict} inequality $p(\x) < p(\y)$. That is, if $G(S)$ contains a cycle, then the
   assumption cannot be fulfilled.
\end{proof}


 To detect cycles, we make use of a basic theorem from graph theory:
\begin{theorem}
\label{graph_corollary}
Let $G$ be a directed graph with at least one edge and assume that $\degp(v) = \degm(v)$ for every vertex $v$. Then there exists a directed cycle.
\end{theorem}
{Here $\degp(v)$  and $\degm(v)$
are the numbers of outgoing and ingoing edges of $v$, respectively.}

\begin{proof}[Proof of Theorem~\ref{thm:main}]
To prove \textbf{irreducibility}, we prove that for all $S\subset \Omega^n$ the graph
   $G(S)$ as defined in Lemma~\ref{lem:reduction}
   contains a cycle. More specifically, we will apply Theorem~\ref{graph_corollary} after showing that for all $S \subset
   \Omega^n$ the graph $G(S)$ has the property 
\begin{equation}
\label{eq:in_eq_out}
\degp(v) = \degm(v)  \enspace, v \in \Omega^n \enspace.
\end{equation}
From the existence of a cycle {for all $S \subset \Omega^n$}, we know
that no {$S \subset \Omega^n$ can be a closed set}
by applying Lemma~\ref{lem:reduction}. 
{Having shown that}
there is
no proper subset of the state space that is a closed set,
{applying Theorem~\ref{thm:closedset} yields}
that the Markov chain is irreducible. 


We prove the property by induction. For the $\textbf{base case}$, we consider  singleton subsets $
\{\x'\}$ and the corresponding edge set:
\begin{equation}
\label{eq:singleton}
   E(\{\x'\}) = \lbrace \big( f_{\leq i-1}(\vec{\x'}), f_{\leq
     i}(\vec{x}') \big), 
        \big(f_{\geq i}(\vec{x}'), f_{\geq i+1}(\vec{x}')\big) \,|\, i=1,\dots n \rbrace \enspace.
 \end{equation}
 

\kb{These edges form a cycle, which corresponds exactly to the contradiction arising
from the sequence of inequalities discussed in the proof of Lemma~\ref{lemma:singeltonSet}. 
Each node of the cycle has exactly one incoming and one outgoing edge, and thus  property \eqref{eq:in_eq_out} holds.}

Let us now assume that $\forall S\subset \Omega^n$  with $|S|=k\ge 1$, for $G(S)=(\Omega^n, E(S))$ induced by \eqref{eq:proof.graph} via  Lemma \ref{lemma1} 
property \eqref{eq:in_eq_out} holds.
In the \textbf{induction step} we now  show {that \ci{then} for all}
$S'\subset \Omega^n $  with $|S'|=k+1$
{property \eqref{eq:in_eq_out} also holds}  for $G(S')=(\Omega^n, E(S'))$.

For each $S'$ it holds $S' = S \cup \{\x'\}$ for some $S\subset \Omega^n, |S|=k$ and $\x' \in \Omega^n \setminus S $. Given $G(S)=(\Omega^n, E(S))$ and $G(\{\x'\}))=(\Omega^n, E(\{\x'\}))$, the edge set $E(S')$ 
can  be constructed as follows. Let $I=\{i\,|\, f_{i}(\x') \in S; i=1,\dots,n\}$:

{\textbf{Case 1}: Assume that $I=\emptyset$, that is, for all $i =1,\dots, n$ it holds $f_i(\x') \notin  S$, or equivalently 
 $f_i(\x') \in  \Omega^n \setminus S$, and therefore
 $f_i(\x') \in  \ci{\Omega^n \setminus S'}$.
Directly from the definitions we have
\begin{multline}
E(S) = \lbrace \big( f_{\leq i-1}(\vec{x}), f_{\leq i}(\vec{x}) \big),
\big(f_{\geq i}(\vec{x}), f_{\geq i+1}(\vec{x})\big) \,|\, \x  \in S \wedge f_i(\x) \in \Omega^n \setminus S \rbrace=\\
\lbrace \big( f_{\leq i-1}(\vec{x}), f_{\leq i}(\vec{x}) \big),
\big(f_{\leq i-1}(\xbar{\vec{x}}), f_{\leq i}(\xbar{\vec{x}}) \big) \,|\, \x 
       \in S \wedge f_i(\x) \in \Omega^n \setminus S \rbrace 
       \enspace
     \end{multline}
     and 
\begin{multline}
E(\{\x'\}) = \lbrace \big( f_{\leq i-1}(\vec{\x'}), f_{\leq i}(\vec{x}') \big),
       \big(f_{\geq i}(\vec{x}'), f_{\geq i+1}(\vec{x}')\big) \,|\,
       i=1,\dots n \rbrace=\\
       \lbrace \big( f_{\leq i-1}(\vec{\x'}), f_{\leq i}(\vec{x}') \big),
       \big(f_{\leq i}(\xbar{\vec{x}'}), f_{\leq i+1}(\xbar{\vec{x}'})\big) \,|\,
       i=1,\dots n \rbrace
       \enspace.
     \end{multline}
     If $\xbar{\vec{x}'} \in S$ then \ci{the} edges in $E(\{\x'\})$ are
     already in $E(S)$ and $E(S \cup \{x'\}) = E(S)$. By assumption
     $\vec{x}' \notin S$. Thus, if $\xbar{\vec{x}'} \notin S$ then
      $E(S)$ and $E(\{\x'\})$ are disjoint and
\begin{multline}
 E(S')=\lbrace \big( f_{\leq i-1}(\vec{x}), f_{\leq i}(\vec{x}) \big),
       \big(f_{\geq i}(\vec{x}), f_{\geq i+1}(\vec{x})\big) \,|\, \x \in S\cup \{\x'\} \wedge f_i(\x) \in \Omega^n \setminus \ci{S'} \rbrace\\= E(S) \cup E(\{\x'\}) \enspace.
     \end{multline}
     In both cases $\degp(v) = \degm(v)$ for all vertices of $G(S')$.}

\textbf{Case 2}: Let $I=\{i\,|\, f_{i}(\x') \in S; i=1,\dots,n\} \neq\emptyset$.
The set 
$E(S')$ can be constructed by removing edges from $E(S) \cup E(\{\x'\})$. This is
illustrated in Figure~\ref{fig:ind}.
We need to remove the edges 
\begin{equation}
R_i^S=\lbrace \big( f_{\leq i-1}(\vec{x}), f_{\leq i}(\vec{x}) \big),
\big(f_{\geq i}(\vec{x}), f_{\geq i+1}(\vec{x})\big) \,|\, \x  \in S \wedge f_i(\x) = \x'\rbrace\subseteq E(S)
\end{equation}
for $i\in I$, because $\x'$ is in the complement of $S$, but not in the complement of $S'$ and thus $R_i^S \nsubseteq E(S')$,\footnote{For $i\not\in I$ we have $R^S_i=\emptyset$, as $f_i(\x) = \x'$ would imply $\x= f_i(\x')$.} 
as well as the edges from 
\begin{equation}R_i^{\{\x'\}}=\lbrace \big( f_{\leq i-1}(\vec{\x'}), f_{\leq i}(\vec{x}') \big),
       \big(f_{\geq i}(\vec{x}'), f_{\geq
         i+1}(\vec{x}')\big)\rbrace\subseteq E(\{\x'\})\label{newprrof:a}
\enspace\end{equation}
for each $i\in I$,
because $f_{i}(\x') \notin  \Omega^n \setminus S'$  implies
$R_i^{\{\x'\}}\nsubseteq E(S')$.
We can rewrite 
\begin{align}
R_i^S&=\lbrace \big( f_{\leq i-1}(f_{i}(\vec{x'})), f_{\leq i}(f_{i}(\vec{x'})) \big),
\big(f_{\geq i}(f_{i}(\vec{x'})), f_{\geq i+1}(f_{i}(\vec{x'}))\big)  \rbrace\notag\\
&= \lbrace  (f_{\leq i}(\vec{x'}), f_{\leq i-1}(\vec{x'})), ( f_{\geq i+1}(\vec{x'}), f_{\geq i}(\vec{x'})) \rbrace \label{newprrof:b} \enspace.
\end{align}

Comparing \eqref{newprrof:a} to \eqref{newprrof:b} shows
that for each edge from $E(\{\x'\})$  not included in $E(S')$ we find
an reverse edge from $E(S)$ not included in $E(S')$.
From this it
follows that $\degp(v) = \degm(v)$ for all vertices of $G(S')$.
Thus, for all proper subsets $S\subset \Omega^n$ the graph $G(s)$
contains a cycle because of Theorem~\ref{graph_corollary}.
This implies that $S$ cannot be a closed set by Lemma~\ref{lem:reduction}. 
Thus, no proper subset of the state space is a closed set and the Markov chain is irreducible by
Theorem~\ref{thm:closedset}.

To prove \textbf{aperiodicity} of an irreducible Markov chain we only
need to identify one aperiodic state. There is one state $\x\in\Omega^n$
with $p(\x)\ge p(\y)$ for all $\y  \in\Omega^n$. By definition of the
modified Metropolis operator \eqref{eq:modified} 
(and also the standard operator \eqref{eq:standard}), there is always
the possibility that after reaching $\x$ the chain also stays in the
state $\x$. Thus, the state $\x$ has period one and the Markov chain is aperiodic.
\end{proof}

\section{Experiments}\label{sec:examples}
Our main result is theoretical and guarantees the irreducibility of the Metropolis chain for all  binary Gibbs
modelsfor which the modified and the vanilla Metropolis algorithm are identical.  This contributes to the basic understanding of the general Metropolis algorithm, which has many applications in machine learning.
Still, it is interesting to look at scenarios
where  the modified Metropolis algorithm differs from the original
from the practical point of view.
Therefore, we performed experiments on Ising models (with uniform coupling strength and no external field)  \citep{ising1925beitrag,HistoryLenzIsing}, which belong to the class of models for which the condition in Corollary~\ref{eq:standard} does not hold.
We selected Ising models because they are conceptually simple and have
been used to demonstrate convergence problems of the Metropolis
algorithm before.

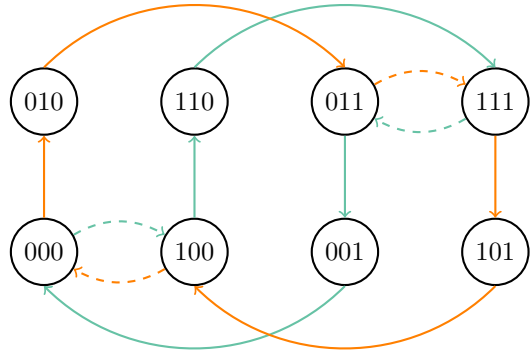
\begin{wrapfigure}{r}{.5\textwidth}
  \centering
  \begin{tikzpicture}[
		scale=2,
		state/.style={circle,draw=black,thick},
	]
	
	\node[state] (z000) at (0,0) {000};
	\node[state] (z100) at (1,0) {100};
	\node[state] (z010) at (0,1) {010};
	\node[state] (z110) at (1,1) {110};
	\node[state] (z001) at (2,0) {001};
	\node[state] (z101) at (3,0) {101};
	\node[state] (z011) at (2,1) {011};
	\node[state] (z111) at (3,1) {111};
	
	\definecolor{greenish}{RGB}{102,194,165}
	\definecolor{orangish}{RGB}{252,141,9}
	
	\draw[->, greenish, thick, dashed] (z000) to [bend left = 30] (z100);
	\draw[->, greenish, thick] (z100) to (z110);
	\draw[->, greenish, thick] (z110.north) to  [bend left = 45] (z111.north);
	
	\draw[->, greenish, thick, dashed] (z111) to [bend left = 30] (z011);
	\draw[->, greenish, thick] (z011) to (z001);
	\draw[->, greenish, thick] (z001.south) to [bend left = 45]  (z000.south);
	
	\draw[->, orange, thick, dashed] (z100) to [bend left = 30] (z000);
	\draw[->, orange, thick] (z000) to (z010);
	\draw[->, orange, thick] (z010.north) to  [bend left = 45] (z011.north);
	
	\draw[->, orange, thick, dashed] (z011) to [bend left = 30] (z111);
	\draw[->, orange, thick] (z111) to (z101);
	\draw[->, orange, thick] (z101.south) to [bend left = 45]  (z100.south);
	
\end{tikzpicture} 
  \caption{Example for the induction step in the proof of Theorem~\ref{thm:main}.
We consider the state space $\{0,1\}^3$, the set  $S =  \{ (0,0,0) \}$, and $\x' = (1,0,0)$.
Both $S$ and $\{\x'\}$ are singleton sets, so their graphs $G(S)$ and
$G(\{\x'\})$ form one cycle each, corresponding to the cycle of
inequalities described \ci{in the proof of Lemma~\ref{lemma:singeltonSet}}. For each edge
dropped from $G(S)$ there is an edge dropped from $G(\{\x'\})$ in the
opposite direction. This means that for each vertex $v$ of $G(S')$, the
property $\degp(v) = \degm(v)$ still holds and the remaining edges
form  disjunct cycles, in this case two cycles. The figure shows
    $G(  \{ (0,0,0) \} )$ in green, $G( \{ (1,0,0) \} )$  in orange,
     $G(  \{ (0,0,0),  (1,0,0) \} )$ with solid lines and the edges that
     are removed with dashed lines. \label{fig:ind} }
\end{wrapfigure}
We applied the original Metropolis algorithm, the modified Metroplis
algorithm as well as Gibbs sampling 
to small two-dimensional ferromagnetic Ising models (with no external field), see Section~\ref{app:ising} in the supplementary appendix. We measure the
 speed of convergence of a Markov chain by the spectral
gap, the difference between the largest eigenvalue, which is always
one for stochastic matrices, and the second largest eigenvalue modulus
$\lambda_2$ of the transition matrix (corresponding to $\vec{T}$).  A
spectral gap of value $1$ means that the Markov chain converges in a
single step, while a value of $0$ means that it is not an ergodic
chain. A value between $0$ and $1$ means that the Markov chain
converges exponentially towards the stationary distribution with the
error going down with  $\mathcal{O}(|\lambda_2|^k$), where $k$ numbers
the steps.\footnote{{Calculating the spectral gap limits our analysis to quite small models, because the size of the transition matrix depends on the
number of states, which grows exponentially with the number of
variables. However, the spectral gap quantifies the speed of
convergence \kb{exactly}.}}
We compared the convergence speed of the three methods by constructing
the full transition matrices and calculating the spectral gap directly
for $3\times 3$ {Ising} models with different coupling strengths $J$.  We use both Ising
models with periodic and non-periodic boundary conditions for the
lattice structure.  We tested the chessboard order and a linear order
for sweeping sites, see Section~\ref{app:counter} in the supplementary appendix.

There was no difference for the convergence speed between
chessboard and linear update order, therefore we only present the results for the 
chessboard order.
In general, the convergence profile of Gibbs sampling is relatively
simple. Without coupling ($J=0$) the sites are independent and the Markov chain ends up in the correct distribution in a single step. The stronger the coupling gets, the slower Gibbs sampling converges.
The graphs for the modified Metropolis operator contain a
discontinuity at $J = 0$ by definition. With zero coupling the modified Metropolis operator behaves exactly like the Gibbs operator, converging in a single step corresponding to a spectral gap of $1$.  For small non-zero coupling strengths the  modified Metropolis operator exhibits the same behavior as the classical Metropolis operator, both converge very slowly with a spectral gap close to $0$.

\begin{figure}[ht]
  \centering
  \includegraphics[width=.49\columnwidth]{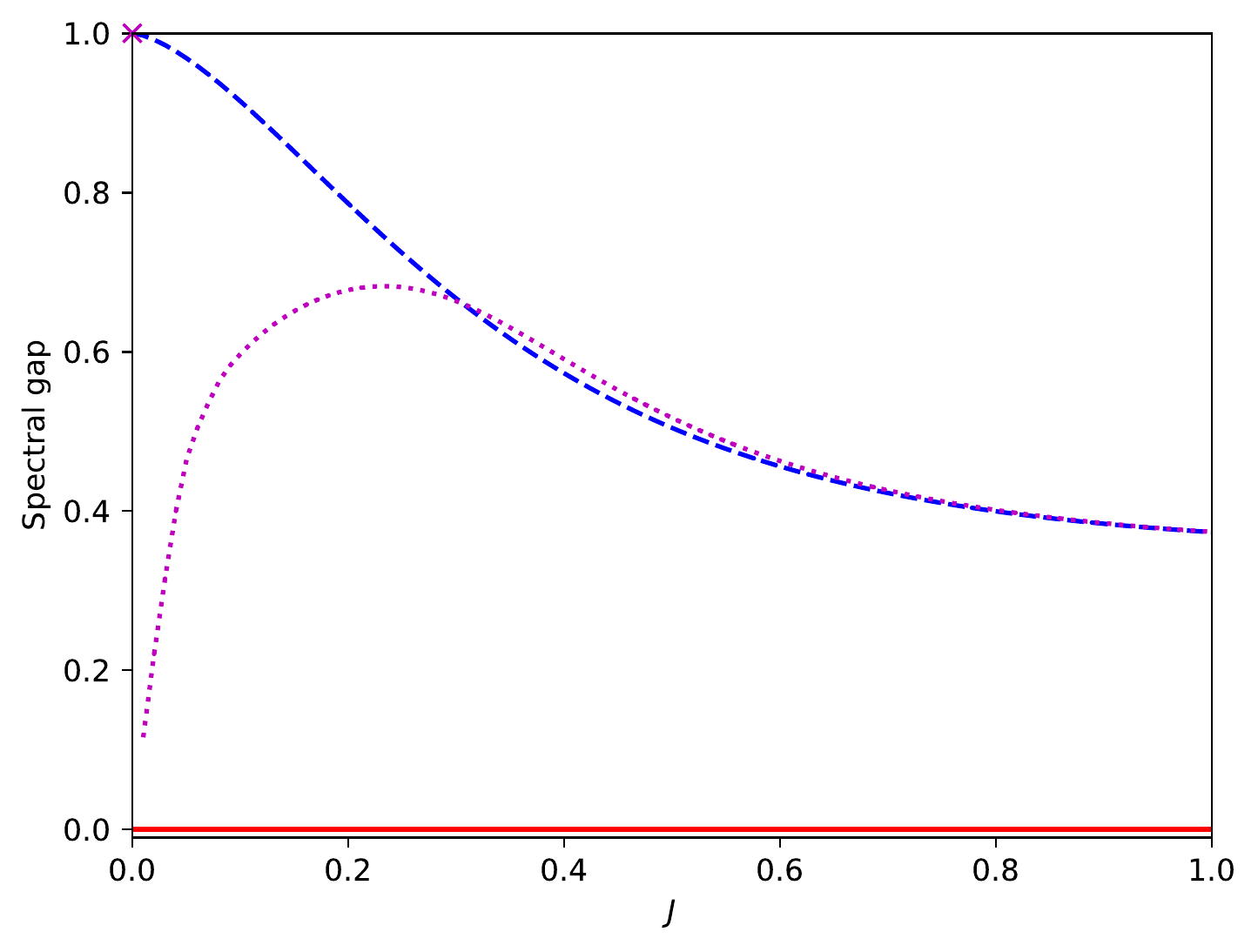}
  \includegraphics[width=.49\columnwidth]{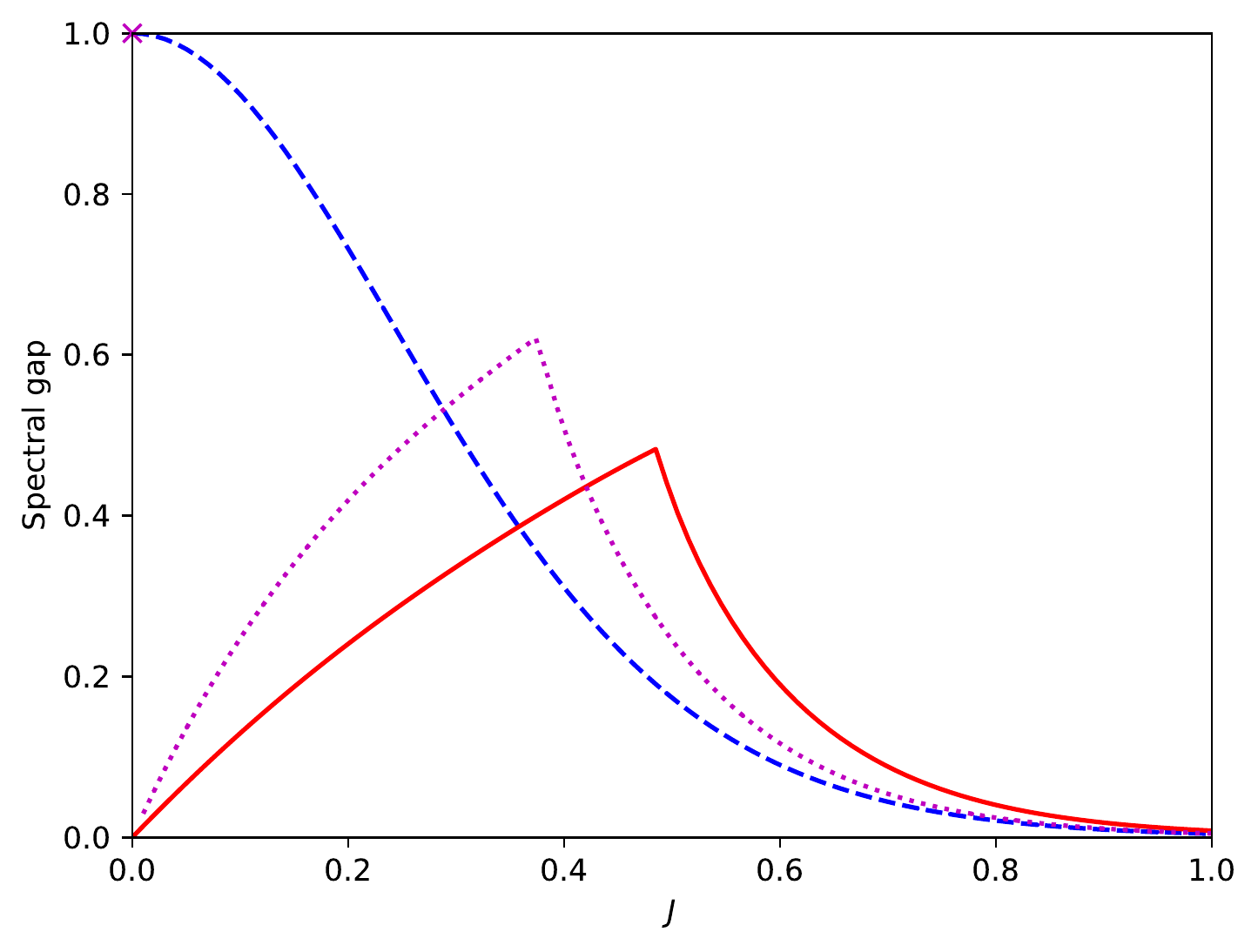}
  \vspace{-1ex}
\caption{The spectral gap of the transition matrix for a $3\times 3$
  Ising model with periodic boundary (left) and non-periodic (right) conditions
  for the Gibbs (blue, dashed)
  as well as 
  Metropolis (red) and modified Metropolis operator (magenta, dotted, {discontinuity at ${J=0}$ marked with $\times$})  in dependence on the coupling strength $J$.  \label{fig:periodic}   
}
\end{figure}

The results for {experiments} in which  the classical Metropolis algorithm does not
converge are shown in the left plot in Figure~\ref{fig:periodic}: {For a $3\times 3$ Ising model
with {a} periodic lattice structure}
the spectral gap for Metropolis sampling is $0$, irrespective of coupling
strength. 
{That is, the Markov
chain induced by the  classical Metropolis algorithm with both chessboard and linear update order is reducible for this model.} 
In contrast, the modified Metropolis algorithm  converges in 
this setting. After the discontinuity at $J = 0$, the convergence speed
improves quickly with rising coupling strength until it reaches a
point where it converges slightly faster than Gibbs sampling, but approaches the same convergence speed at very strong couplings.
The right plot in Figure \ref{fig:periodic} depicts the results for a case where both
the classical and modified Metropolis converge very slowly for models
with weak coupling, converge better with rising coupling strength
until they surpass Gibbs sampling. After they reach a point with
maximal convergence speed, the convergence speed decreases again and
all three sampling methods approach the same spectral gap for very
high coupling strengths.
Although the overall shapes of the curves are very similar {for both versions of the Metropolis algorithm}, the
spectral gap of the modified Metropolis operator rises quicker,
reaches its maximum earlier and then also decreases earlier and
faster. 
{In summary,} there are models with low $J$ where Gibbs
sampling converges the fastest, models with medium $J$ where the
modified Metropolis algorithm  converges the fastest, and models with
high $J$ where the classical algorithm converges the fastest. But overall the modified and classical Metropolis operator show a very similar behavior.

\section{Conclusions}\label{sec:conclusions}
 
 There has been a gap between theory and practice in MCMC sampling using the
 Metropolis algorithm.
The necessary mathematical conditions ensuring an
ergodic Markov chain do in general not hold for multivariate binary distributions
for the most popular {variant} of the Metropolis algorithm, which updates 
{the states of the random variables} in a
fixed order, but this has not stopped the community from 
employing it. 
This paper
  shows that only a very small modification of the
Metropolis algorithm is necessary to produce an ergodic
Markov chain that guarantees \kb{convergence
for all multivariate binary distributions}.
Moreover, our results allow to identify the scenarios in which the standard Metropolis 
algorithm
  for multivariate binary distributions  and fixed-order updates may not converge
 and those in which it is safe to use. Specifically, reducibility
   only occurs if there exist states that differ only in one
   variable and have the same probability  under the stationary
   distribution. If no such states exist, our convergence proof
   can directly be transfered to the standard Metropolis algorithm -- 
   and this will almost surely be the case  for typical models such as Boltzmann machines and spin glasses
   with (initial) parameters drawn from a continuous distribution.
   This insight closes the gap between practice and (Markov chain) theory for the Metropolis algorithm with fixed-order updates for multivariate binary distributions.
 

{We measured the spectral gap for Ising models (with uniform
  coupling strength and no external field),
  where}
the standard Metropolis algorithm  with fixed update order
is not guaranteed to converge.
{In settings where the}
standard Metropolis MCMC 
converges,  
the modified Metropolis algorithm has a similar convergence speed
profile as the original algorithm, but behaves slightly more like Gibbs sampling.
When looking at Ising models where
the standard Metropolis algorithms does
not converge, the modified version  does. Here
Gibbs sampling gives better results for low coupling
strengths, while the new Metropolis operator performs slightly better
for medium coupling strengths.

In summary, 
we  see no argument why the newly proposed 
transition operator should not be the default setting for Metropolis MCMC with
fixed-order updates of binary variables.

\ifpreprint
\begin{ack}
CI acknowledges support by the Villum Foundation through the project Deep Learning and Remote Sensing for Unlocking Global Ecosystem Resource Dynamics (DeReEco).
\end{ack}
\fi
\bibliography{proof.bib}
\bibliographystyle{abbrvnat}

\newpage
\appendix

\renewcommand\thefigure{\thesection.\arabic{figure}} 
\renewcommand\thetable{\thesection.\arabic{table}} 
\renewcommand{\theequation}{\thesection.\arabic{equation}}
\renewcommand{\thetheorem}{\thesection.\arabic{theorem}}

\section{Ising model}\label{app:ising}

An Ising model describes a
$d$-dimensional lattice with interactions between neighboring sites,
where the coupling strength $J$ controls how strong a tendency there
is for neighbouring sites to be the same. We have
$\Omega=\lbrace -1, +1 \rbrace$ and denote the set of pairs of indexes
of neighboring sites by $\Lambda$.  The probability of the state of an Ising
model is given
by
\[
p(\vec{x}) = \frac{e^{-E(\vec{x})}}{\sum_{\vec{x}\in\Omega^n} e^{-E(\vec{x})}} \text{\qquad with\qquad}
E(\vec{x}) = - J \sum_{(i,j) \in \Lambda} x_i x_j\enspace.
\]

\section{Examples where the Metropolis algorithm does not
  converge}\label{app:counter}

\begin{figure*}[htb]
\hfill\begin{tikzpicture}[site/.style={circle, minimum size=0.8cm, draw=black, thick} ]
\draw[step=1.0,black,thin] (-0.5,-0.5) grid (3.5,3.5);
\draw (0,3) node [site, fill=blue!10] {$\mathbbf[1.5]{+}$};
\draw (1,3) node [site, fill=white] {$\mathbbf[1.5]{+}$};
\draw (2,3) node [site, fill=blue!10] {$\mathbbf[1.5]{+}$};
\draw (3,3) node [site, fill=white] {$\mathbbf[1.5]{+}$};
\draw (0,2) node [site, fill=white] {$\mathbbf[1.5]{-}$};
\draw (1,2) node [site, fill=blue!10] {$\mathbbf[1.5]{-}$};
\draw (2,2) node [site, fill=white] {$\mathbbf[1.5]{-}$};
\draw (3,2) node [site, fill=blue!10] {$\mathbbf[1.5]{-}$};
\draw (0,1) node [site, fill=blue!10] {$\mathbbf[1.5]{+}$};
\draw (1,1) node [site, fill=white] {$\mathbbf[1.5]{+}$};
\draw (2,1) node [site, fill=blue!10] {$\mathbbf[1.5]{+}$};
\draw (3,1) node [site, fill=white] {$\mathbbf[1.5]{+}$};
\draw (0,0) node [site, fill=white] {$\mathbbf[1.5]{-}$};
\draw (1,0) node [site, fill=blue!10] {$\mathbbf[1.5]{-}$};
\draw (2,0) node [site, fill=white] {$\mathbbf[1.5]{-}$};
\draw (3,0) node [site, fill=blue!10] {$\mathbbf[1.5]{-}$};
\end{tikzpicture}\hfill
\begin{tikzpicture}[site/.style={circle, minimum size=0.8cm, draw=black, thick} ]
\draw[step=1.0,black,thin] (-0.5,-0.5) grid (3.5,3.5);
\draw (0,3) node [site, fill=white] {$\mathbbf[1.5]{-}$};
\draw (1,3) node [site, fill=blue!10] {$\mathbbf[1.5]{+}$};
\draw (2,3) node [site, fill=white] {$\mathbbf[1.5]{-}$};
\draw (3,3) node [site, fill=blue!10] {$\mathbbf[1.5]{+}$};
\draw (0,2) node [site, fill=blue!10] {$\mathbbf[1.5]{-}$};
\draw (1,2) node [site, fill=white] {$\mathbbf[1.5]{+}$};
\draw (2,2) node [site, fill=blue!10] {$\mathbbf[1.5]{-}$};
\draw (3,2) node [site, fill=white] {$\mathbbf[1.5]{+}$};
\draw (0,1) node [site, fill=white] {$\mathbbf[1.5]{-}$};
\draw (1,1) node [site, fill=blue!10] {$\mathbbf[1.5]{+}$};
\draw (2,1) node [site, fill=white] {$\mathbbf[1.5]{-}$};
\draw (3,1) node [site, fill=blue!10] {$\mathbbf[1.5]{+}$};
\draw (0,0) node [site, fill=blue!10] {$\mathbbf[1.5]{-}$};
\draw (1,0) node [site, fill=white] {$\mathbbf[1.5]{+}$};
\draw (2,0) node [site, fill=blue!10] {$\mathbbf[1.5]{-}$};
\draw (3,0) node [site, fill=white] {$\mathbbf[1.5]{+}$};
\end{tikzpicture}\hfill\phantom{}

\bigskip

\hfill\begin{tikzpicture}[site/.style={circle, minimum size=0.8cm, draw=black, thick} ]
\draw[step=1.0,black,thin] (-0.5,-0.5) grid (3.5,3.5);
\draw (0,3) node [site, fill=blue!10] {$\mathbbf[1.5]{-}$};
\draw (1,3) node [site, fill=white] {$\mathbbf[1.5]{-}$};
\draw (2,3) node [site, fill=blue!10] {$\mathbbf[1.5]{-}$};
\draw (3,3) node [site, fill=white] {$\mathbbf[1.5]{-}$};
\draw (0,2) node [site, fill=white] {$\mathbbf[1.5]{+}$};
\draw (1,2) node [site, fill=blue!10] {$\mathbbf[1.5]{+}$};
\draw (2,2) node [site, fill=white] {$\mathbbf[1.5]{+}$};
\draw (3,2) node [site, fill=blue!10] {$\mathbbf[1.5]{+}$};
\draw (0,1) node [site, fill=blue!10] {$\mathbbf[1.5]{-}$};
\draw (1,1) node [site, fill=white] {$\mathbbf[1.5]{-}$};
\draw (2,1) node [site, fill=blue!10] {$\mathbbf[1.5]{-}$};
\draw (3,1) node [site, fill=white] {$\mathbbf[1.5]{-}$};
\draw (0,0) node [site, fill=white] {$\mathbbf[1.5]{+}$};
\draw (1,0) node [site, fill=blue!10] {$\mathbbf[1.5]{+}$};
\draw (2,0) node [site, fill=white] {$\mathbbf[1.5]{+}$};
\draw (3,0) node [site, fill=blue!10] {$\mathbbf[1.5]{+}$};
\end{tikzpicture}\hfill
\begin{tikzpicture}[site/.style={circle, minimum size=0.8cm, draw=black, thick} ]
\draw[step=1.0,black,thin] (-0.5,-0.5) grid (3.5,3.5);
\draw (0,3) node [site, fill=blue!10] {$\mathbbf[1.5]{+}$};
\draw (1,3) node [site, fill=white] {$\mathbbf[1.5]{-}$};
\draw (2,3) node [site, fill=blue!10] {$\mathbbf[1.5]{+}$};
\draw (3,3) node [site, fill=white] {$\mathbbf[1.5]{-}$};
\draw (0,2) node [site, fill=white] {$\mathbbf[1.5]{+}$};
\draw (1,2) node [site, fill=blue!10] {$\mathbbf[1.5]{-}$};
\draw (2,2) node [site, fill=white] {$\mathbbf[1.5]{+}$};
\draw (3,2) node [site, fill=blue!10] {$\mathbbf[1.5]{-}$};
\draw (0,1) node [site, fill=blue!10] {$\mathbbf[1.5]{+}$};
\draw (1,1) node [site, fill=white] {$\mathbbf[1.5]{-}$};
\draw (2,1) node [site, fill=blue!10] {$\mathbbf[1.5]{+}$};
\draw (3,1) node [site, fill=white] {$\mathbbf[1.5]{-}$};
\draw (0,0) node [site, fill=white] {$\mathbbf[1.5]{+}$};
\draw (1,0) node [site, fill=blue!10] {$\mathbbf[1.5]{-}$};
\draw (2,0) node [site, fill=white] {$\mathbbf[1.5]{+}$};
\draw (3,0) node [site, fill=blue!10] {$\mathbbf[1.5]{-}$};
\end{tikzpicture}\hfill\phantom{}

\caption{Counter-example with chessboard update order. We assume a periodic Ising model with uniform coupling-strengths, where both of the dimensions of the Ising model have an even number of sites.
If the site activations form a striped pattern and are updated in a
chessboard pattern (here indicated by the coloring of the nodes, where
first the nodes of one color are updated, then the nodes of the other
color), they will never escape the striped pattern, switching deterministically from horizontal to vertical stripes after updating half the variables, switching to the complementary horizontal striped pattern after a full update step and switching back in another full update step.}
\label{fig:counterexamples_chess}
\end{figure*}
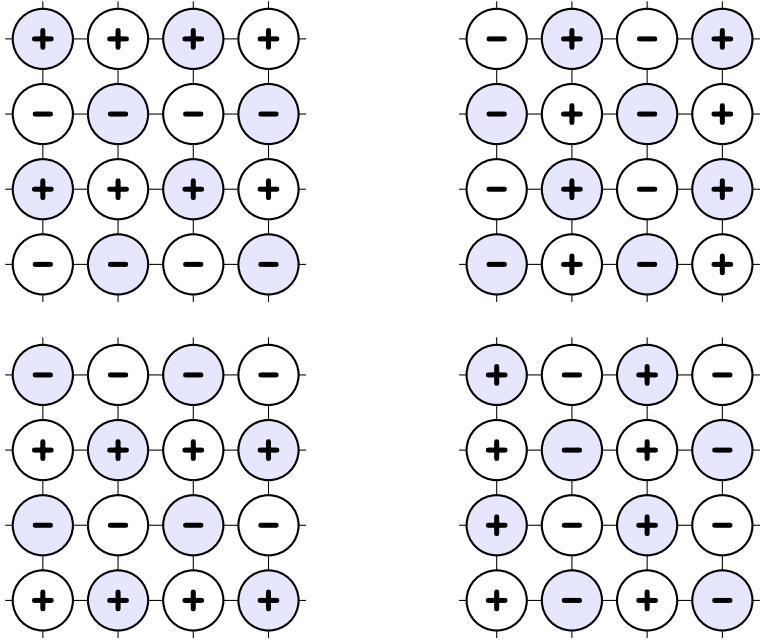

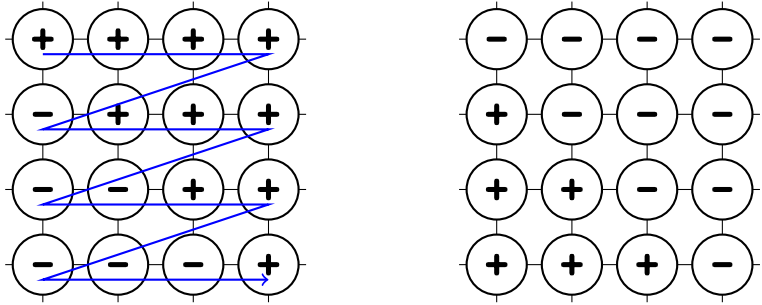
\begin{figure}[htb]
\hfill\begin{tikzpicture}[site/.style={circle, minimum size=0.8cm, draw=black, thick} ]
\draw[step=1.0,black,thin] (-0.5,-0.5) grid (3.5,3.5);
\draw (0,3) node [site, fill=white] {$\mathbbf[1.5]{+}$};
\draw (1,3) node [site, fill=white] {$\mathbbf[1.5]{+}$};
\draw (2,3) node [site, fill=white] {$\mathbbf[1.5]{+}$};
\draw (3,3) node [site, fill=white] {$\mathbbf[1.5]{+}$};
\draw (0,2) node [site, fill=white] {$\mathbbf[1.5]{-}$};
\draw (1,2) node [site, fill=white] {$\mathbbf[1.5]{+}$};
\draw (2,2) node [site, fill=white] {$\mathbbf[1.5]{+}$};
\draw (3,2) node [site, fill=white] {$\mathbbf[1.5]{+}$};
\draw (0,1) node [site, fill=white] {$\mathbbf[1.5]{-}$};
\draw (1,1) node [site, fill=white] {$\mathbbf[1.5]{-}$};
\draw (2,1) node [site, fill=white] {$\mathbbf[1.5]{+}$};
\draw (3,1) node [site, fill=white] {$\mathbbf[1.5]{+}$};
\draw (0,0) node [site, fill=white] {$\mathbbf[1.5]{-}$};
\draw (1,0) node [site, fill=white] {$\mathbbf[1.5]{-}$};
\draw (2,0) node [site, fill=white] {$\mathbbf[1.5]{-}$};
\draw (3,0) node [site, fill=white] {$\mathbbf[1.5]{+}$};
\begin{scope}[shift={(0,-0.2)}]
\draw[->, draw=blue,  thick]
(0,3) -- (1,3) -- (2,3) -- (3,3) -- 
(0,2) -- (1,2) -- (2,2) -- (3,2) -- 
(0,1) -- (1,1) -- (2,1) -- (3,1) -- 
(0,0) -- (1,0) -- (2,0) -- (3,0);
\end{scope}
\end{tikzpicture}\hfill\begin{tikzpicture}[site/.style={circle, minimum size=0.8cm, draw=black, thick} ]
\draw[step=1.0,black,thin] (-0.5,-0.5) grid (3.5,3.5);
\draw (0,3) node [site, fill=white] {$\mathbbf[1.5]{-}$};
\draw (1,3) node [site, fill=white] {$\mathbbf[1.5]{-}$};
\draw (2,3) node [site, fill=white] {$\mathbbf[1.5]{-}$};
\draw (3,3) node [site, fill=white] {$\mathbbf[1.5]{-}$};
\draw (0,2) node [site, fill=white] {$\mathbbf[1.5]{+}$};
\draw (1,2) node [site, fill=white] {$\mathbbf[1.5]{-}$};
\draw (2,2) node [site, fill=white] {$\mathbbf[1.5]{-}$};
\draw (3,2) node [site, fill=white] {$\mathbbf[1.5]{-}$};
\draw (0,1) node [site, fill=white] {$\mathbbf[1.5]{+}$};
\draw (1,1) node [site, fill=white] {$\mathbbf[1.5]{+}$};
\draw (2,1) node [site, fill=white] {$\mathbbf[1.5]{-}$};
\draw (3,1) node [site, fill=white] {$\mathbbf[1.5]{-}$};
\draw (0,0) node [site, fill=white] {$\mathbbf[1.5]{+}$};
\draw (1,0) node [site, fill=white] {$\mathbbf[1.5]{+}$};
\draw (2,0) node [site, fill=white] {$\mathbbf[1.5]{+}$};
\draw (3,0) node [site, fill=white] {$\mathbbf[1.5]{-}$};
\end{tikzpicture}\hfill\phantom{}
\caption{Counter-example with left-to-right, top-to-bottom scanning
  order. For a periodic Ising model, where both of the dimensions have
  the same number of sites, if the site activations form a triangular
  pattern as shown in the left figure, they will flip between two
  possibles states if updated in the order indicated by the arrow.}
\label{fig:counterexamples_linear}
\end{figure}

Real-world counter-examples, where the Metropolis acceptance function
with fixed-order updates leads to a reducible Markov chain, were found, for
instance,  by \citet{friedberg1970}, who simulated $4\times 4$
two-dimensional Ising models and discovered {that they can be}
``locked in configurations'', but  did not make the connection to
Markov chain theory and irreducibility. Two examples are shown in
Figure~\ref{fig:counterexamples_chess}
and Figure~\ref{fig:counterexamples_linear}.
Both examples consider \emph{periodic} models, that is, we
we assume periodic boundary
conditions (i.e., sites on one edge of the lattice are connected to
sites on the opposite edge).  

\end{document}